\documentclass{article}
\usepackage{ijcai19}
\usepackage{times}

\usepackage{algpseudocode,algorithm}
\usepackage{mkolar_definitions}
\usepackage{amsmath}
\usepackage{amsthm}
\usepackage{amsbsy}
\usepackage{amsfonts}
\usepackage{multirow}
\usepackage{graphicx}
\usepackage{epsfig}
\usepackage{epstopdf}
\usepackage{amssymb}
\usepackage{subfigure}
\usepackage{wrapfig}
\usepackage{float}
\usepackage{color}
\usepackage{url}
\newcommand*\samethanks[1][\value{footnote}]{\footnotemark[#1]}

\title{Lexicographically Ordered Multi-Objective Clustering}
\author{Sainyam Galhotra\thanks{The first two authors have contributed equally to this work.} \and Sandhya Saisubramanian\samethanks \and Shlomo Zilberstein 
	\affiliations
	College of Information and Computer Sciences, 	University of Massachusetts Amherst \emails
	\{sainyam, saisubramanian, shlomo\}@cs.umass.edu \\	
}

\begin{document}
		
\maketitle
\begin{abstract}
	We introduce a rich model for multi-objective clustering with lexicographic ordering over objectives and a slack. The slack denotes the allowed multiplicative deviation from the optimal objective value of the higher priority objective to facilitate improvement in lower-priority objectives. We then propose an algorithm called Zeus to solve this class of problems, which is characterized by a \emph{makeshift} function.
	The makeshift fine tunes the clusters formed by the processed objectives so as to improve the clustering with respect to the unprocessed objectives, given the slack. We present makeshift for solving three different classes of objectives and analyze their solution guarantees. Finally, we empirically demonstrate the effectiveness of our approach on three applications using real-world data. 
\end{abstract}

\section{Introduction}
Identifying graph clusters, which are groups of similar or related entities~\cite{jain1999data}, is being increasingly employed for data-driven decision making in high-impact applications such as health care~\cite{haraty2015enhanced} and urban mobility~\cite{kumar2016understanding,saisubramanian2015risk}. Clustering with multiple objectives~\cite{law2004multiobjective,handl2007evolutionary} helps improve robustness of the solution and has proven to be beneficial in many applications such as resource sharing~\cite{chen2011parallel}, fairness~\cite{fairlet}, and team formation~\cite{farhadi2012teamfinder}. 
For example, consider a group of six friends who want to carpool to work in two cars (Figure~\ref{fig:illus-1}). When clustering for carpooling, it is important to minimize the maximum distance traveled by the driver ($o_1$) and balance the cluster sizes ($o_2$).

\begin{figure}[t]
	\centering
	\subfigure[Problem]{\includegraphics[scale=.3]{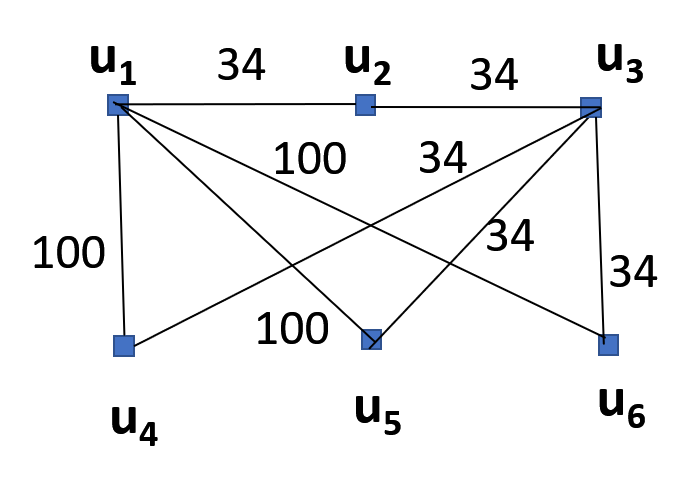}\label{ex-problem}}
	\subfigure[$o_1$]{\includegraphics[scale=.3]{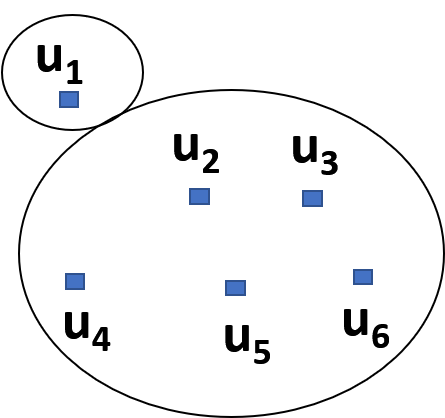}\label{ex-o1}}
	\subfigure[$o_1,o_2$]{\includegraphics[scale=.3]{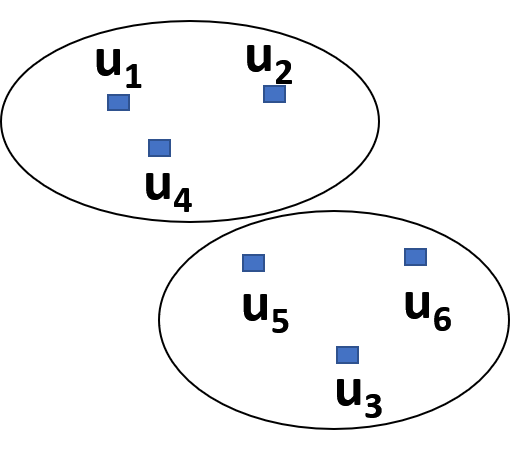}\label{ex-moc}}
	\subfigure[$o_1>o_2$]{\includegraphics[scale=.32]{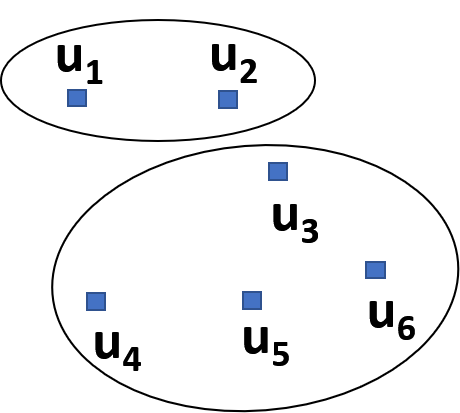}\label{ex-lmoc}}
	\quad
	\subfigure[Objective values]{\includegraphics[scale=.38]{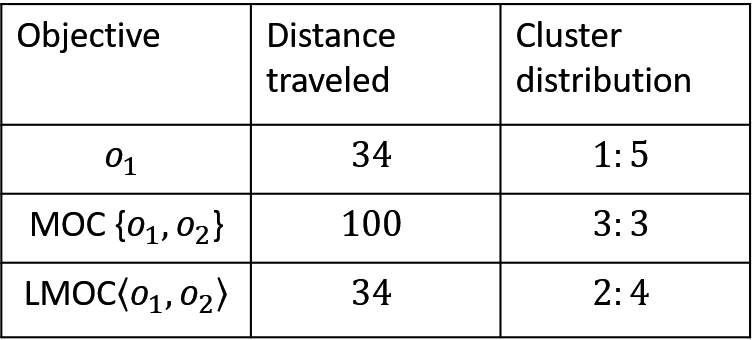}\label{tab:objectives}}
	\caption{An example of single and multi-objective clustering for the carpooling problem. Each edge weight denotes the pairwise distances between the nodes and all pairs which are not connected by an edge have a distance of 67.} 
	\label{fig:illus-1}
\end{figure}

Existing techniques that support multi-objective clustering (MOC) either leverage a scalarization function, which combines the multiple objectives into a single objective, or find clusters in parallel for each objective and combine the results using different approaches such as fitness function~\cite{jiamthapthaksin2009framework,veldt2018correlation,handl2007evolutionary,pizzuti2018evolutionary,saha2018exploring}. Finding a suitable scalarization is non-trivial due to the large space of Pareto optimal solutions that may need to be explored~\cite{handl2007evolutionary,WZMaaai15,zhou2011multiobjective}. Many algorithms employ a heuristic approach to prune the space of Pareto optimal solutions, making it difficult to provide any theoretical guarantees on the results. When combining solutions from solving multiple objectives in isolation, it is not clear how the solution to one objective affects another objective since the clusters formed may be arbitrarily worse with respect to other objectives. For the carpooling example, clusters formed by optimizing independently for $o_1$ (using~\citeauthor{kcenter2approx}~(\citeyear{kcenter2approx})) and $MOC \{o_1,o_2\}$ (using~\citeauthor{balancedkc}~(\citeyear{balancedkc})) are shown in Figures~\ref{ex-o1} and~\ref{ex-moc}, and the corresponding objective values are tabulated in Figure~\ref{tab:objectives}. In single objective clustering, the solution is far from optimal for $o_2$. The distance traveled is much larger with MOC since it optimizes both the objectives together. 

We address these concerns by considering a lexicographic ordering over objectives, which offers a natural way to describe optimization problems~\cite{rangcheng2001discounted,WZMaaai15}. This is motivated by the observation that many multi-objective problems are characterized by an inherent lexicographic ordering over objectives, which offers a principled approach to evaluate candidate solutions. In fact, for the carpooling scenario, clusters that optimize for \emph{both} $o_1$ and $o_2$ in the order $o_1\!>\!o_2$ (Figure~\ref{ex-lmoc}) achieve the best trade-off between the two objectives.

%

We introduce \emph{Relaxed Lexicographic Multi-Objective Clustering} (RLMOC), a generalized model that supports clustering with any finite number of objectives and is characterized by a slack variable. The lexicographic ordering enforces a preference over objectives that are satisfied by the solution. Strict lexicographic ordering often reduces the space for forming clusters that satisfy lower-priority objectives, which we alleviate by using a slack.
The slack, for each objective, is a multiplicative approximation factor denoting the upper limit on the acceptable loss in the solution quality from the optimal, thus offering more flexibility for clustering. For example, in clustering for supply-demand matching, there is always a trade-off among optimizing for distance, load balance, and cost. When optimizing the distance is less critical, allowing for a slack helps improve load balancing and cost. RLMOC generalizes (lexicographic) multi-objective clustering and single-objective clustering, since their solutions can be achieved by appropriately modifying the slack. 

We propose the Zeus algorithm that solves (relaxed) LMOC problem by sequentially processing the different objectives to form clusters. This is facilitated by a \emph{makeshift} subroutine that processes the clusters formed by the previous objective so as to improve the clustering with respect to the current objective, as long as the loss in solution quality does not violate the slack. By varying the list of objectives, their ordering, and the allowed slack, a wide range of problems can be efficiently represented by our model. In this paper, we discuss in detail the makeshifts for resource sharing, fairness, team formation, and K-center objectives, and analyze their theoretical guarantees.

Our primary contributions are: (i) introducing the lexicographic multi-objective clustering with slack (Section 2); (ii) presenting Zeus algorithm, makeshift functions for solving various classes of problems, and analyzing their theoretical guarantees (Section 3); and (iv) empirical results on three domains with different combinations of objectives (Section 4).

\section{Problem Formulation}
Consider a collection of $n$ points $V\!=\{v_1,v_2,\ldots,v_n\}$, along with a pairwise distance metric $d:V\times V\!\rightarrow\!\mathbb{R}$. Let $H\!=\!G(V,E,d)$ be a graph with $E\subseteq V \times V$ capturing pairwise relationships such as `do $u,v\in V$ know each other?'. Let $\chi$ denote a function that maps each node to the different set of attributes\footnote{$H$ comprises of various attributes like $E$, $d$ and $\chi$. Only the attributes optimized by the objectives are required to be known.}. Given a graph instance $H$ and an integer $k$, the goal is to construct clusters $\mathcal{C}\!=\!\{C_1,C_2,\ldots, C_k\}$ that  partition $V$ into $k$ disjoint subsets by optimizing an objective function. Given a graph $H$ and a set of clusters $\mathcal{C}$, the objective function ($o$) returns an objective value as a real number, $o(H,\mathcal{C})\rightarrow \mathbb{R}$, which helps compare different clustering techniques. $C(u)$ denotes the cluster corresponding to $u\!\in\!V$. 

Our work focuses on using a lexicographic collection of these objectives to optimize the set of clusters obtained. Given an ordered set of objectives $\mathcal{O}\!=\!\langle o_1,o_2,\ldots,o_r\rangle$, the lexicographic preference enforces the following priority: $o_1\!>\!o_2\!>\!\ldots\!>o_r$. We now define a mechanism to compare two different sets of clusters to identify a lexicographically superior set of clusters and use it to define the Lexicographic Multi-Objective Objective Clustering (LMOC).

\begin{definition}[Lexicographically Superior]
Given two sets of clusters $\mathcal{C}_1$ and $\mathcal{C}_2$ that optimize for a lexicographically ordered set of objectives $\mathcal{O}\!=\!\langle o_1,o_2,\ldots o_r\rangle$ over a graph $H$, 
$\mathcal{C}_1$ is 
lexicographically superior to $\mathcal{C}_2$ ($\mathcal{C}_1 > \mathcal{C}_2$) if there exists $0\leq t\leq r$ such that  $o_j(H,\mathcal{C}_1)=o_j(H,\mathcal{C}_2)$, $\forall 0<j< t$ and $o_t(H,\mathcal{C}_1) > o_t(H,\mathcal{C}_2)$ whenever $o_t$ is a maximization objective (and the opposite if $o_t$ is a minimization objective).
\end{definition}

\begin{definition}[Lexicographic Multi-Objective Clustering: LMOC($H,k,\mathcal{O}$)]
Given a graph $H=G(V,E,d)$, an ordered set of objectives $\mathcal{O}=\langle o_1,o_2,\ldots,o_r\rangle$, and an integer $k$, the goal is to find a set of k-clusters $\mathcal{C}=\{C_1,C_2,\ldots,C_k\}$ such that there does not exist any other set of k-clusters which are lexicographically superior to $\mathcal{C}$.
\end{definition}
The LMOC generalizes single objective clustering problem and satisfies the following properties:
\begin{itemize}
	\item Optimizing the same objective multiple times is equivalent to optimizing for the same objective once,  $LMOC(H,k,\langle o_1,o_2,o_1\rangle) = LMOC(H,k,\langle o_1,o_2\rangle)$;
	\item The objective value of optimal clusters returned by  $LMOC(H,k,\langle o_1,o_2\rangle)$  for $o_2$ is not less than the objective value of optimal clusters for single objective problem that minimizes $o_2$.
	\item The  clusters returned by LMOC are sensitive to the order in which objective functions are considered, $LMOC(H,k,\langle o_1,o_2\rangle) \neq LMOC(H,k,\langle o_2,o_1\rangle)$; and
	\item$LMOC(H,k,\langle o_1,-o_1\rangle) = LMOC(H,k,\langle o_1\rangle)$.
\end{itemize}

Given the complexity of identifying a global superior set of clusters, we define Relaxed Lexicographic Multi-Objective Clustering problem (RLMOC) that is characterized by an ordered set of \emph{slack} values, ${\Delta}\!=\!\langle\delta_1,\delta_2,\ldots,\delta_r\rangle$ with $\delta_i\!\geq\!0$. The slack values denote the multiplicative approximation factor corresponding to each objective in the lexicographic order. Therefore, a solution $\mathcal{C}$ is considered to be valid iff $o_j(H,\mathcal{C})\geq \delta_j \times o_j(H,\mathcal{C}^*)$, $\forall j\le r$ when $o_j$ is a maximization objective (and  $o_j(H,\mathcal{C})\leq \delta_j o_j(H,\mathcal{C}^*)$ if $o_j$ is a minimization objective), where $\mathcal{C}^*$ is a globally lexicographically superior set of clusters. The goal is to return any single set of clusters from the space of possible solutions. Every RLMOC is an LMOC when $\delta_i\!=\!1, 1\!\leq i \leq r$ and the properties described for LMOC can be translated for the relaxed version as well. Given the NP-hardness of identifying optimal clusters (globally superior) with respect to classical objectives (k-center, k-median, k-means), higher values of $\delta_i$ enable trading solution quality for faster computations. In the following section, we describe an approach to solve RLMOC.

\section{Solution Approach}
In this section, we present the Zeus algorithm to solve the RLMOC problem. Given $H$, the algorithm initializes each node to be present in its own separate cluster and sequentially processes the objective functions (Algorithm~\ref{alg:zeus}). This is achieved by employing a \emph{makeshift} subroutine that processes the previously formed clusters to satisfy the current objective. In each application of the makeshift, the goal is to obtain a set of clusters that do not violate the slack values of any of the processed objectives. The $\texttt{slack\_violated}$ function calculates the objective value of the clustering and estimates if the corresponding slack is violated\footnote{Since calculating the optimal objective value can be NP-hard, theoretical guarantees are leveraged to estimate the optimal value.}. When any of the slack values are violated by the makeshift, the clusters are post-processed using a local search algorithm to improve the violated objective function, without affecting the quality with respect to other objectives. The \texttt{local\_search} function aims to improve the solution by moving one node at a time from its original cluster to any other cluster and terminates when a solution that does not violate the slack is found or when any movement of the nodes does not result in an improvement.
\begin{algorithm}[t]
	\begin{algorithmic}[1]
		\State  $\mathcal{C}\leftarrow $Initialize each node in a separate cluster
		\For{$o_i\in \mathcal{O}$}
		\State $\mathcal{C}\leftarrow$ \texttt{makeshift\_o}$_i$ ($H,\mathcal{C}$)		
		\If{\texttt{slack\_violated}($\mathcal{C},o_i,\delta_i$)}
		\State $\mathcal{C}\leftarrow$ \texttt{local\_search} ($\mathcal{C},\delta_i,o_i$)
		\EndIf
		\EndFor
		\State \Return $\mathcal{C}$;
	\end{algorithmic}
	\caption{\texttt{Zeus}($H,k,\mathcal{O},\Delta$)  \label{alg:zeus}}
\end{algorithm}
Zeus supports any combination of objectives since the makeshift is independent of the sequence of objectives considered. 

\subsection{\emph{Makeshift}\label{sec:makeshift}}
The makeshift is a critical component of Zeus in solving RLMOC. Since the makeshift modifies the clusters formed using previously processed objective to satisfy the current objective, they are naturally dependent on the objective function for efficiency. It is relatively easier to design makeshifts for the \emph{classical} clustering objectives such as K-center, K-median, and K-means. When using a combination of classical and ancillary clustering objectives, the makeshift for the ancillary objective depends on the classical clustering objective as well. In the rest of the paper, we focus on three ancillary objectives that are widely used in real-world applications that benefit from multiple objectives~\cite{zhou2011multiobjective}, each in combination with the K-center objective. A brief description of how our makeshifts can be adapted for other classical clustering objectives is discussed in the appendix.

\subsubsection{k-Center}
It is one of the most widely studied objectives in the literature~\cite{kcenter2approx}, where the goal is to identify $k$ nodes as cluster centers (say $S$, $|S|=k$) and assign each node to the closest cluster center such that the maximum distance of any node from its cluster center is minimized. The objective value is calculated as:
\[o_{kC}(H,\mathcal{C}) = \max_{u\in V} \min_{v\in S} d(u,v).\]

A simple greedy algorithm provides a 2-approximation for the k-center problem and it is NP-hard to find a better approximation factor~\cite{kcenter2approx}. The greedy algorithm initializes each point to be in its own cluster and chooses the first center randomly. In each subsequent iteration, all nodes are assigned to the already identified centers. The node which is farthest from the currently assigned center is selected as the new cluster center. The makeshift algorithm for k-center leverages this to identify the cluster centers. Whenever the input to k-center is a collection of clusters formed by previously processed objectives, the makeshift post-processes these clusters by reassigning nodes such that the set of nodes which were clustered together before processing k-center belong to the same cluster.

\subsubsection{Resource Sharing (RS)}
The objective in resource sharing, $o_{RS}$, is to maximize the number of nodes that have at least one of its \emph{neighbors} in the same cluster. The objective value is calculated as:
\[o_{RS}(H,\mathcal{C}) = \frac{|\{u:\exists v\in C(u), (u,v)\in E\}|}{|V|}.\]
Clustering for RS is widely used in distributed computing and cache management where each compute node in the network is assigned to one of the caches~\cite{motivationrs}. 

For the sake of clarity, we introduce the makeshift (Algorithm~\ref{alg:edgecover}) assuming each node is in its own cluster but this can be modified to work with situations when there are multiple nodes clustered together. The makeshift first iterates over $V$ and constructs a subgraph $H'$ by considering the minimum weight edge for every node (Lines 2-5). The edges $E'$ are a valid edge-cover of the graph $H$ (Lemma~\ref{lem:optimaledge}). The edges are sorted in decreasing order of their weights and the redundant edges are then removed, while ensuring that $E'$ is a valid edge-cover. The set of nodes that belong to the same connected component in $H'$ are considered to belong to the same cluster (Line 12). Hence the clusters generated by Algorithm~\ref{alg:edgecover} are star-shaped with the star centers acting as cluster centers and the maximum length of any path in $H'$ is two (Lemma~\ref{lem:star}). Algorithm~\ref{alg:edgecover} is highly efficient with a run time polynomial in $\vert E\vert$, and hence $O(\vert V\vert^2)$ in the worst case.

\begin{algorithm}[t]
	\begin{algorithmic}
		\State  $E'\leftarrow \phi$, $H'\leftarrow (V,E')$
		\For{$u \in V$}
		\State $v\leftarrow argmin_{v\in V}\{d(u,v) \mid  (u,v)\in E\}$
		\State $E'\leftarrow E'\cup \{(u,v)\}$
		\EndFor
		\State $E' \leftarrow sortDescending(E')$
		\For{$e\in  E'$}
		\If{$E'\setminus \{e\}$ is a valid edge cover}
		\State $E'\leftarrow E'\setminus\{e\}$
		\EndIf
		\EndFor
		\State $\mathcal{C}\leftarrow connected\_components(H')$
		\State \Return $\mathcal{C}$;
	\end{algorithmic}
	\caption{\texttt{makeshift\_RS}$(H,\mathcal{C}$)  \label{alg:edgecover}}
\end{algorithm}
\begin{lemma} Algorithm~\ref{alg:edgecover} produces optimal edge cover that minimizes the maximum weight of any edge in $E'$.\label{lem:optimaledge}
\end{lemma}
\begin{proof}
	To show that Algorithm~\ref{alg:edgecover} produces optimal edge cover that minimizes the maximum weight edge, we first show that it produces a valid edge cover. The algorithm constructs $E'$ with the minimum weight edge for all $v\in V$ (Lines 3-5) and any redundant edge, which does not violate the edge cover condition, is removed (Line 6). Thus, ${E}'$ is a valid edge cover. We now prove by contradiction that ${E}'$ is optimal. Let ${E}^*$ be an edge cover with a smaller value of the maximum weight edge. Hence, there exists an edge $e\!=\!(u,v)\!\in\!{E^*}$ with $d(u,v)\!<\!d(u',v'), \forall (u',v')\!\in {E}'$. However, for the node $u$, the algorithm selected $(u,v')$ for the edge cover (Lines 3-6), which we have shown to be a valid edge cover. Hence, $d(u,v)\!\ge\! d(u,v')$ is a contradiction, proving that Algorithm~\ref{alg:edgecover} produces optimal maximum edge cover.
\end{proof}
\begin{lemma}
The maximum length of any path in $H'(V,E')$, formed by Algorithm \ref{alg:edgecover}, is two.\label{lem:star}
\end{lemma}
\begin{proof}
	We prove by contradiction that the maximum path length in $H'(V,E')$ is two. Upon termination of the algorithm, let there exist a path $P$ with edges: $\{(u_1,u_2),(u_2,u_3),(u_3,u_4),\ldots, (u_t,u_{t+1})\}$ and $t>2$. Removing the edge $(u_2,u_3)$ from $P$ would still result in an edge cover. However, Algorithm~\ref{alg:edgecover} evaluates every edge in $E'$ and only retains the edges if the edge cover is not violated (Lines 9-13). Hence, this is a contradiction as the $E'\setminus \{(u_2,u_3)\}$ cannot be a valid edge cover. Hence, the maximum length of all paths in $H'(V,E')$ is two. 
\end{proof}
The following theorem proves that the set of clusters returned by Algorithm~\ref{alg:edgecover} are optimal with respect to $o_{RS}$ and bounds the solution quality.
\begin{theorem}
	The clustering $\mathcal{C}$ returned by Zeus when $\mathcal{O}=\langle RS,kC\rangle$ has objective values $o_{RS}(H,\mathcal{C})=OPT_{{RS}}$ and $o_{kC}(H,\mathcal{C})=3\ OPT_{\langle {RS},{kC}\rangle}$ in the worst case, where $OPT_o$ is the optimal objective value for the objective function $o$.
	\label{thm:resource}
\end{theorem}
\begin{proof}
	Let $\mathcal{C}=\{C_1,C_2,\ldots, C_k\}$ be the clusters returned by Zeus with $\mathcal{O}=\langle RS,kC\rangle$ and $\mathcal{C}'$ be the clusters generated on processing RS (and $E'$ be the edges in the edge cover). Since $E'$ is the optimal edge cover (by Lemma~\ref{lem:optimaledge}) and the pair of nodes that are connected in the edge cover are present in the same final cluster, the objective value of $\mathcal{C}$ with respect to $RS$ is the optimal objective value. 

Lemma~\ref{lem:optimaledge} shows that the weight of any edge $e\in E'$ is less than the maximum weight of an edge in the optimal solution. Hence, $d(e)\leq OPT_{\mathcal{O}}$, $\forall e\in E'$ and Lemma~\ref{lem:star} guarantees that the clusters generated by RS are star-shaped. When Zeus processes k-center objective, the the distance between any pair of points is less than $2OPT_{kC}$~\cite{kcenter2approx}. Using the second property of LMOC, we know that $2OPT_{kC}\leq 2OPT_{\mathcal{O}}$. It then iterates over the leaves in the stars of $\mathcal{C}'$  and  assigns them to the same cluster as that of the center of its star in $\mathcal{C}$. Using triangle inequality, the distance of these leaves from the center of their new cluster is less than the sum of `distance between  the leaf and center of the star (say a)' and  `distance between center of the star and its cluster center (say b)'. Since (a) $\leq OPT_{\mathcal{O}}$ and (b) $\leq 2OPT_{\mathcal{O}}$. Hence, the maximum distance of any point from its cluster center is less than $3OPT_{\mathcal{O}}$.
\end{proof}

This shows that the slack is not violated for $\mathcal{O}=\langle RS,kC \rangle$ whenever $\delta_{RS}\le 1$ and $\delta_{kC}\ge 3$. Additionally,  $\delta_{RS}>1$ is infeasible as the objective value can never be greater than the optimal value and $\delta_{kC}< 2$ is also infeasible due to NP-hardness 
of approximating k-center problem. Hence, the slack is violated only when $2\leq \delta_{kC}\leq 3$ in which case the local search technique helps improve the solution.

\subsubsection{Fairness (F)}
Minimizing bias to improve fairness is gaining increased attention as it is critical for many real-world settings. However, fairness in clustering remains under-explored. Let each node in the graph $H$ have a sensitive attribute, say color which can `Blue' (B) or `Purple' (P). Given such a characteristic, recent work has focused on forming clusters such that every cluster has equal fraction of nodes with `Blue' attribute~\cite{fairlet}.
This is a form of group fairness, studied in the literature~\cite{fairnesssg}. Another setting studied in the literature considers individual fairness, where two nodes possessing similar attributes but different color should not be discriminated~\cite{fairnesssg}. 
We consider a fairness objective, $o_F$, which ensures that each node from the minority group (say `Blue') is matched to at least one \emph{neighbor} from the majority group (`Purple') and all the matched pairs (denoted by $E'$) belong to the same cluster. The objective value is calculated as:
\[o_{F}(H,\mathcal{C}) = \frac{|\{u:\!u\!\in\!B, (u,v)\!\in\!E', v\in C(u)\cap P\}|}{|B|}.\]
\begin{algorithm}[t]
	\begin{algorithmic}[1]
		\For{$r^* =$ binary search in $[0,max\{d(u,v)\}]$}
		\State $E'\leftarrow \phi$
		\For{$(u,v) \in E$}
		\If{$d(u,v)\leq r^*$}
		\State $E'\leftarrow E'\cup(u,v)$
		\EndIf
		\EndFor
		\State $\mathcal{C}\leftarrow \texttt{Matching(G(V,E'))}$
		\EndFor
		\State \Return $\mathcal{C}, E'$;
	\end{algorithmic}
	\caption{\texttt{makeshift\_F}($H,\mathcal{C}$) \label{alg:fairness}}
\end{algorithm}

Algorithm~\ref{alg:fairness} describes the mechanism to match the nodes with `Blue' color to the nodes with `Purple' color. Let $r^*$ denote the optimal value of the maximum distance between any pair of matched vertices. The optimal distance is initialized to the maximum distance between any pair of vertices, $r^* = \max_{(u,v)\in E} \{d(u,v)\}$, and is refined iteratively. The algorithm constructs an unweighted bipartite graph with `Blue' nodes on one side ($B$) and `Purple' on the other side ($P$) and a pair is connected if the corresponding edge distance is less than $r^*$. It then performs a maximum bipartite matching by adding a source node, $s$, and a sink node, $t$. The nodes in $B$ are connected to $s$ and the nodes in $P$ are connected to $t$, with an edge of unit capacity.
Executing bipartite matching on this instance guarantees a that every node in $B$ is matched with some node in $P$. In each subsequent iteration, $r^*$ is updated by performing binary search, which helps quickly identify the smallest $r^*$ that guarantees finding an optimal matching. 

\begin{theorem}
The clustering $\mathcal{C}$ returned by Zeus when $\mathcal{O}=\langle F,kC\rangle$ has objective values $o_{F}(H,\mathcal{C})=OPT_{{F}}$ and $o_{kC}(H,\mathcal{C})=3\ OPT_{\langle {F},{kC}\rangle}$ in the worst case, where $OPT_o$ is the optimal objective value for the objective function $o$.\label{thm:fairness}
\end{theorem}
\begin{proof}
	Let $\mathcal{C}=\{C_1,C_2,\ldots, C_k\}$ be the clusters returned by Zeus with $\mathcal{O}=\langle F,kC\rangle$ and $\mathcal{C}'$ be the clusters generated on processing F (and $E'$ be the edges in the matching returned). Since $E'$ is the optimal maximum matching  and the pair of nodes that are connected in the matching are present in the same final cluster, the objective value of $\mathcal{C}$ with respect to $F$ is the optimal objective value. 
	
	Since we perform a binary search to identify the smallest value of $r^*$ that returns the maximum matching, the weight of any edge $e\in E'$ is not greater than the maximum weight of an edge in the optimal solution. Hence, $d(e)\leq OPT_{\mathcal{O}}$, $\forall e\in E'$. When Zeus processes k-center objective, the distance between any pair of points is less than $2OPT_{kC}$ [Vazirani, 2013]. Using the second property of LMOC, we know that $2OPT_{kC}\leq 2OPT_{\mathcal{O}}$. It then iterates over the nodes in the components of $\mathcal{C}'$ and assigns all matched nodes to the same cluster in $\mathcal{C}$. Using triangle inequality, the distance of these leaves from the center of their new cluster is less than the sum of `distance between the leaf and center of the star (say a)' and  `distance between center of the star and its cluster center (say b)'. Since (a) $\leq OPT_{\mathcal{O}}$ and (b) $\leq 2OPT_{\mathcal{O}}$. Hence, the maximum distance of any point from its cluster center is less than $3OPT_{\mathcal{O}}$.
\end{proof}
This shows that the slack is not violated for $\mathcal{O}=\langle F,kC \rangle$ whenever $\delta_{F}\le 1$ and $\delta_{kC}\ge 3$. Additionally,  $\delta_{F}>1$ is infeasible as the objective value can never be greater than the optimal value and $\delta_{kC}< 2$ is also infeasible due to NP-hardness 
of approximating k-center problem. Hence, the slack is violated only when $2\le \delta_{kC}\leq 3$ in which case the local search technique helps improve the solution. 

\subsubsection{Team Formation (TF)}
This objective is motivated by applications that require forming teams (clusters) such that certain attributes (experts in different fields) are equally represented across all clusters, irrespective of their connectivity with other nodes.
%
Consider a scenario where each node has an attribute $X$ such that $X(u)=True$ denotes that $u$ is an expert in $X$, and a non-expert otherwise. We consider $X$ to be a binary variable but it can be extended to work when there are multiple attributes, each having multiple values. The team formation objective aims to form clusters with equal fraction of experts; each cluster has (approximately) $\frac{|X|}{k}$ nodes from $X$, with $X\!\subseteq\!V$ denoting the experts. The objective value is:

\[o_{TF}(H,\mathcal{C}) = \frac{\max_{C\in \mathcal{C}}|\{u:u\in C,  X(u)=True\}|} {\min_{C\in \mathcal{C}}|\{u:u\in C, X(u)=True\}|}.\]

In order to handle the team formation objective along with k-center objective, Algorithm~\ref{alg:teamformation} first performs constrained k-center on the set of vertices that are experts, $X(u)\!=\!True$. This step ensures that the k clusters generated are of equal size. 
When $X\!=\!V$, this objective is equivalent to generating balanced clusters (example in Figure~\ref{fig:illus-1}), for which the current best solution is a 4-approximation of the clusters on $V$~\cite{balancedkc}. Every node in $V\setminus X$ is assigned to the cluster corresponding to the closest node in $X$. The time complexity of Algorithm~\ref{alg:teamformation} is $O(\vert V \vert k)$.

\begin{theorem}
The clustering $\mathcal{C}$ returned by Zeus when $\mathcal{O}=\langle TF,kC\rangle$ has objective values $o_{TF}(H,\mathcal{C})\!=\!OPT_{{TF}}$ and $o_{kC}(H,\mathcal{C})\!=\!10\  OPT_{\langle {TF},{kC}\rangle}$ in the worst case, where $OPT_o$ is the optimal objective value for the objective function $o$. \label{Thrm:TF}
\end{theorem}
\begin{proof}
First, we try to estimate the optimal distance of points in $X$ from their corresponding centers. Then, we evaluate the distance of points in $V\setminus X$ from their corresponding centers.
	Let $\mathcal{C}^*$ be the optimal set of clusters that optimize $TF$ and $kC$ objectives with the k-center objective value $r^*$. Hence the pairwise distance between any pair of nodes in $X$ that belong to the same cluster is $\le 2r^*$. Restricting $C^*$ to the set of nodes in $X$ is a valid solution with respect to $TF$ objective. Hence, the set of optimal clusters, on the nodes of $X$ will have k-center objective value $O\le 2r^*$. Using latest result for balanced k-center problem, which is 4-approximation~\cite{balancedkc}, we get  $d(u,c(u))\le 4O \le 8r^*,\ \forall u,c(u)\in X$, where $c(u)$ is the center of the cluster corresponding $u$. 
	
	Additionally, $\mathcal{C}^*$ ensures that $\forall u\in V\setminus X$, there exists some node in $v\in X$ having a distance less than $2r^*$. Since Algorithm~\ref{alg:teamformation} identifies the closest node in $X$ (Line 3), $d(u,v)\le 2r^*,\ u\in V \text{ and } v\in X$. Using triangle inequality, the distance of any node $v$ from the corresponding center is less than $d(v,c(v))\leq d(u,v)+d(u,c(u)) = 8r^*+ 2r^* = 10r^*$.
\end{proof}
\begin{algorithm}[t]
	\begin{algorithmic}[1]
		\State $\mathcal{C}\leftarrow $ balanced k-center on $X$ 
		\For{$u\in V$}
		\State $v\leftarrow \arg \min\{d(u,v)\mid v\in X\}$
		\State $C\leftarrow C\cup\{u\} \ \mid v\in C$
		\EndFor
		\State \Return $\mathcal{C}$;
	\end{algorithmic}
	\caption{\texttt{makeshift\_TF} ($H,\mathcal{C}, X\subseteq V$) \label{alg:teamformation}}
\end{algorithm}
This shows that the slack is not violated for $\mathcal{O}=\langle TF,kC \rangle$ whenever $\delta_{TF}\le 1$ and $\delta_{kC}\!\ge\!10$. Furthermore, $\delta_{TF}>1$ is infeasible as the objective value cannot be greater than the optimal value and $\delta_{kC}< 2$ is infeasible as it is NP-hard to get a better approximation~\cite{kcenter2approx}. In general, if an algorithm that guarantees $\alpha-$approximation for equal cluster k-center algorithm can be devised, then Algorithm~\ref{alg:teamformation} is guaranteed to provide an approximation ratio of $2\alpha+2$. Assigning a node $v\in V\setminus X$ to the closest cluster center improved the solution quality empirically, even though it does not alter the theoretical guarantee.  We employ this optimization in our experiments.

\begin{figure*}
	\centering
	\includegraphics[width=7in, height=1.5in]{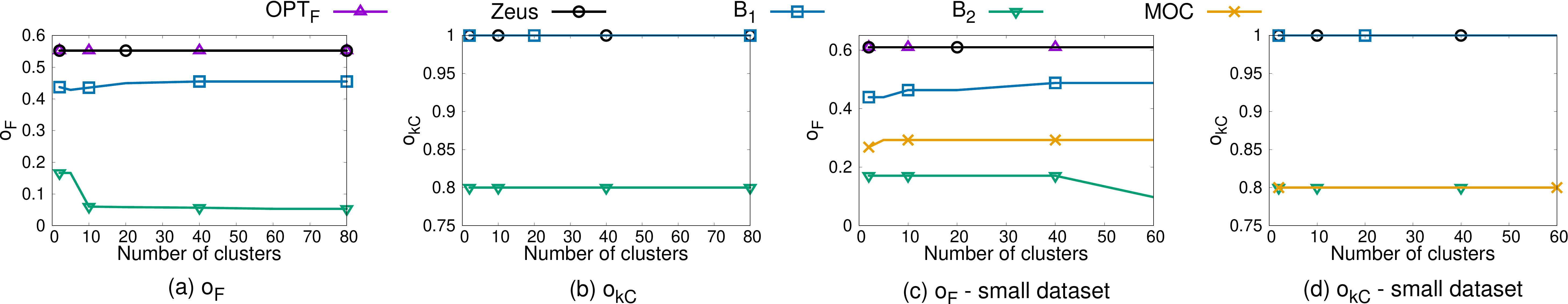}\label{results-tf-1000}
	\caption{Solution quality of various approaches corresponding to fairness and k-center objectives, $\langle F,KC\rangle$.}	\label{fig-fairness}
	\vspace{-7pt}
\end{figure*}

\section{Experimental Results}
We conduct extensive experiments to evaluate Zeus and the proposed makeshifts on three real world datasets with the objectives discussed in the earlier section. Fairness (F), resource sharing (RS), and Team formation (TF) objectives are considered as $o_1$ with K-center (KC) as $o_2$, with the lexicographic order $o_1 > o_2$. The $\langle F, KC\rangle$ objectives are evaluated on the Pokec social network dataset\footnote{https://snap.stanford.edu/data/soc-Pokec.html}. The goal is to form clusters such that for every female in a cluster, there is at least one male neighbor in the same cluster, while optimizing for K-center. The $\langle RS, KC\rangle$ objectives are evaluated on the academic conference dataset\footnote{https://core.ac.uk/services\#dataset}, to identify conferences that can co-occur or be co-located. The $\langle TF, KC\rangle$ objectives are evaluated on the adult dataset\footnote{https://archive.ics.uci.edu/ml/machine-learning-databases/adult/}. The goal is to form teams such that nodes with ``tech-suppport'' attribute are equally distributed across clusters, while optimizing for K-center. For the resource sharing application, the distances are estimated using an embedding in Euclidean space. For fairness and team formation, Jaccard distances are used.

We compare the results produced by Zeus with that of three baselines: ($B_1$) a greedy algorithm that optimizes $o_1$ independently; ($B_2$) optimizing k-center objective, $o_2$ (using~\cite{kcenter2approx}); and (MOC) a greedy approach that optimizes for both the considered objectives, with equal weight to each objective. The results are compared across different value of $k\!\geq\!2$ and different slack values. Unless otherwise specified, all algorithms were implemented by us in Python using the networkx library on a 8GB RAM laptop and the reported results are on 1000 nodes.

 
\begin{figure}[t]
	\centering
	\includegraphics[width=3.2in, height=1.4in]{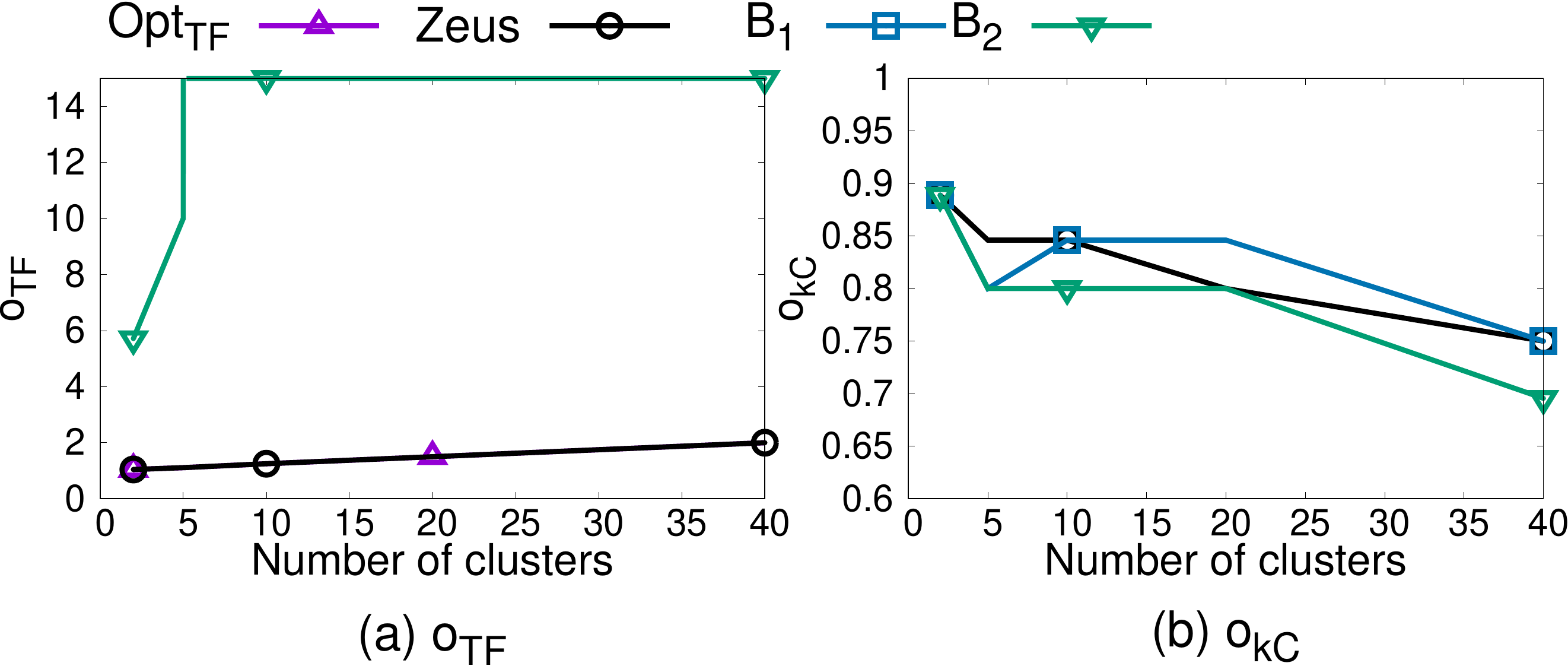}\label{tf_merged}
	\caption{Solution quality of various approaches corresponding to team formation and k-center objectives, $\langle TF,KC\rangle$.}
	\label{fig-teamformation}
\end{figure}

\begin{figure}[t]
	\centering
		\includegraphics[width=3.2in, height=1.4in]{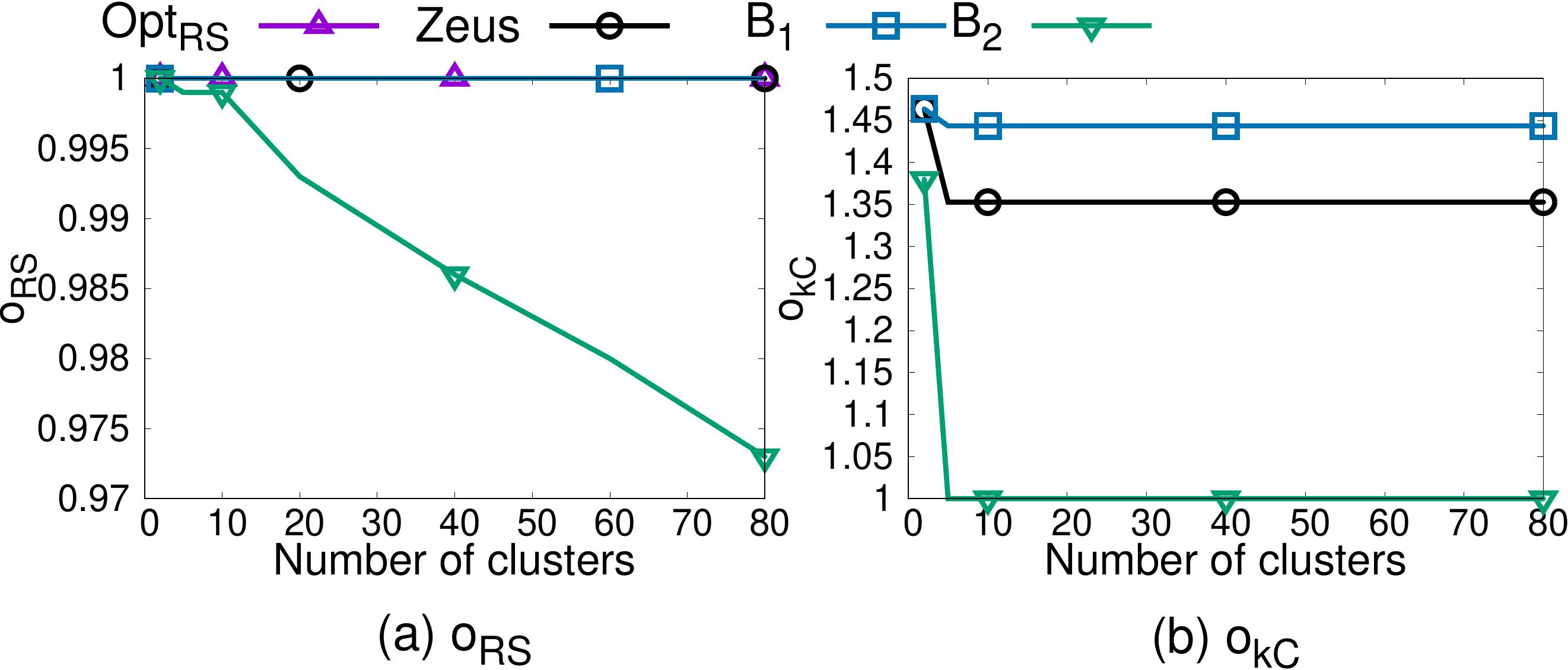}\label{rs_merged}
	\caption{Solution quality of various approaches corresponding to resource sharing and k-center objectives, $\langle RS,KC\rangle$.}
	\label{fig-rs}
\end{figure}
\begin{figure}[t]
	\centering
	\includegraphics[width=3.2in, height=1.4in]{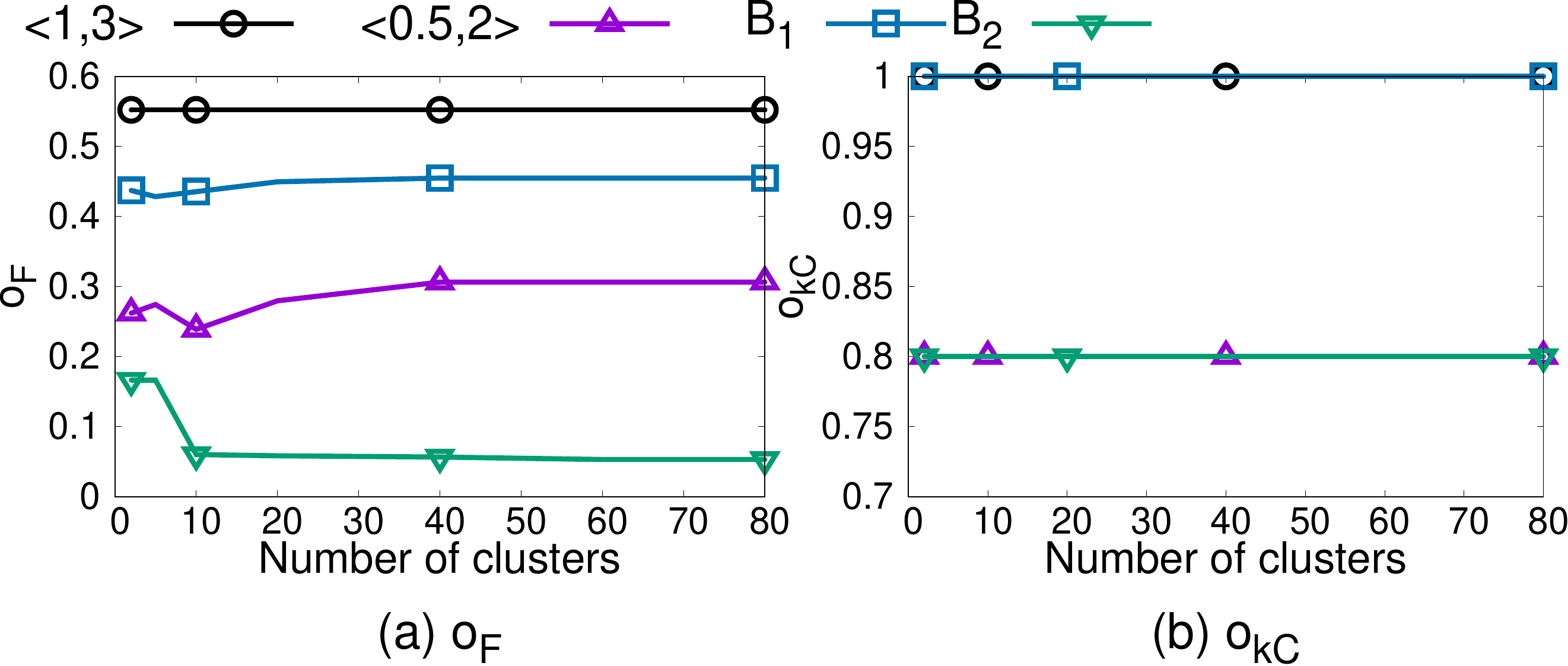}\label{slack_merged}
	\caption{Effect of slack on $\langle F,KC\rangle$.}
\end{figure}

\subsection{Discussion}
\textbf{Solution Quality }  Figure~\ref{fig-fairness} compares the performance of Zeus with the three baselines for $\mathcal{O}\!=\!\langle F, KC \rangle$ with slack values $\Delta=\langle 1,3 \rangle$ (as guaranteed by Theorem \ref{thm:fairness}). $Opt_F$ denotes the optimal fairness objective value on the Pokec dataset. It is evident that Zeus achieves optimal value for fairness and its performance with respect to $o_2$ is closer to that of $B_2$, which optimizes for $o_2$ alone.
The MOC baseline (MOC) did not find clusters even after 24 hours on this problem. Therefore, we compare its results on a smaller subset of this dataset with 100 nodes (3c, 3d). 
MOC performs well for k-center objective but significantly compromises the solution quality for $o_1$. 
Figure~\ref{fig-teamformation} shows results for $\mathcal{O}\!=\!\langle TF, KC \rangle$ with slack values $\Delta=\langle 1,10 \rangle$ (as guaranteed by Theorem \ref{Thrm:TF}). $Opt_{TF}$ denotes the optimal team formation objective value on the adult dataset. Zeus performs similar to $Opt_{TF}$ for all values of $k$ and $B_2$ provides solutions that are far from optimal for $o_{TF}$. For the KC objective, $B_2$ performs better than Zeus, as expected, and Zeus is better than $B_1$. Although the worst case approximation guarantee of Zeus is 10 times worse than that of the optimal (Thrm.~\ref{Thrm:TF}), it performs better in practice. 
Figure~\ref{fig-rs} shows the results for $\mathcal{O}\!=\!\langle RS, KC \rangle$ with slack values $\Delta=\langle 1,3 \rangle$ (as guaranteed by Theorem \ref{thm:resource}). $Opt_{RS}$ denotes the optimal resource sharing objective value on the conference dataset. It is evident that Zeus performs consistently better than the baselines for all values of $k$. 

MOC did not converge on the full dataset for the team formation and resource sharing objectives but the results on 100 nodes were similar to that of fairness. 
Overall, Zeus consistently performs better than all three baselines, on all data sets in our experiments and for all values of $k$. These experiments demonstrate the effectiveness of Zeus in optimizing multiple objectives, given a lexicographic order.

\vspace{6pt}
\noindent{\textbf{Slack }} Altering the slack from $\langle 1,3\rangle$ to $\langle 0.5,2\rangle$ improves the performance of Zeus on $o_{KC}$. Zeus performs similar to $B_2$ on $o_{KC}$, while performing better than $B_2$ on $o_{F}$ (Figures~6a,~6b
). For the sake of consistency, we consider only feasible slack values. This demonstrates that by increasing the slack corresponding to $o_1$, clustering with respect to lower-priority objectives can be improved. Similar results were observed for $\mathcal{O}\!=\!\langle TF, KC \rangle$ and $\mathcal{O}\!=\!\langle RS, KC \rangle$.

\vspace{6pt}
\noindent{\textbf{Runtime }} In our experiments, the run time of Zeus is linear in $k$ and Zeus took at most 30 minutes to form clusters for all values of $k$ and across all datasets.

\section{Conclusion and Future Work}
We introduce the relaxed multi-objective clustering, a general model for clustering with multiple objectives, given a lexicographic order and slack. By altering the slack and the lexicographic order, a wide range of real-world problems can be efficiently modeled using RLMOC. 
We also present Zeus, an efficient algorithm that processes the different objective functions sequentially and leverages a makeshift subroutine to modify the clusters for a particular objective. Theoretical properties are discussed for the three makeshifts described in the paper. Our empirical results show that Zeus effectively optimizes the objectives, in terms of solution quality and run time. Identifying makeshift for various other objectives is an interesting problem for future work. 

{
\bibliographystyle{named}
\bibliography{moc}
}

\appendix
\section{Modification of Objectives}
In this section, we describe variations of resource sharing and fairness objective functions and how the proposed makeshifts can be modified to work for these new objectives. \\

\noindent \textbf{Resource Sharing (RS)}. The RS objective function described in Sec 3.1 tries to ensure that every node has at least one neighbor in the same cluster. A generalization of RS considers a scenario where every node has at least $\gamma$ neighbors in the same cluster. Let $r^*$ denote the optimal value of the maximum distance between any pair of matched vertices. The optimal distance is initialized to the maximum distance between any pair of vertices, $r^* = \max_{(u,v)\in E} \{d(u,v)\}$, and is refined iteratively. In order to account for this modified objective, we can modify Algorithm 2 such that the edges with weight more than $r^*$ are pruned. If the graph formed by residual graphs is a valid edge cover where each node has degree more than $\gamma$, then $r^*$ is valid. In each subsequent iteration, $r^*$ is updated by performing binary search, which helps quickly identify the smallest $r^*$ that guarantees finding a valid edge cover.

\vspace{8pt}
\noindent \textbf{Fairness}. The fairness objective can be modified to handle applications where the goal is to match $\alpha$ members of minority group to $\beta$ members of majority group. Another modification is to consider more than two groups of members and the goal is be to construct a matching between all pairs of such groups. These variations can he handled easily by the makeshift described in Section 3.1 that calculates b-matching for the nodes in the dataset.
Figure \ref{fig:f2_merged} demonstrates the behavior of Zeus on the modified fairness objective where $2$ nodes from $B$ are matched with $2$ nodes of $P$.

\section{Makeshift for Other Classical Clustering Objectives}
Section 3.1 in the main paper describes the makeshifts for different set of objectives studied in the literature, when employed along with k-center objective. We now describe the variation of all these objectives for \emph{k-median} objective. The algorithm proposed by Vazirani et al. (2013)  is one of the popularly used approaches for performing k-median (kM) clustering. This algorithm can be used as the makeshift for kM. We show ways to adapt the makeshift with respect to resource sharing, fairness, and team formation objectives when applied along with k-median objective.\\

\noindent \textbf{Resource Sharing}. The makeshift proposed in Sec 3.1 for resource sharing and k-center works well for the k-median objective as well. This is because the algorithm returns the optimal edge-cover, which minimizes the maximum  weight of any edge in the edge cover along with the sum of weights of edges in the cover. \\

\noindent \textbf{Fairness}. The makeshift proposed in Sec 3.1 for fairness and k-center can be modified to work with the k-median objective. 
The same algorithm works for k-median with a slight modification that all edges $(u,v)$ are considered while computing the matching and the  weight on each edge $d(u,v)$ acts as the cost of the edge.
With this  construction of the bipartite graph, the minimum cost matching is generated. This matching guarantees that the total distance between any pair of matched vertices in minimized.\\

\noindent \textbf{Team Formation}. Instead of running k-center algorithm on the set of nodes in $X$, we run the k-median algorithm and the makeshift works the same way as described in Algorithm 4.

\begin{figure}[t]
	\centering
	\includegraphics[width=3.2in, height=1.5in]{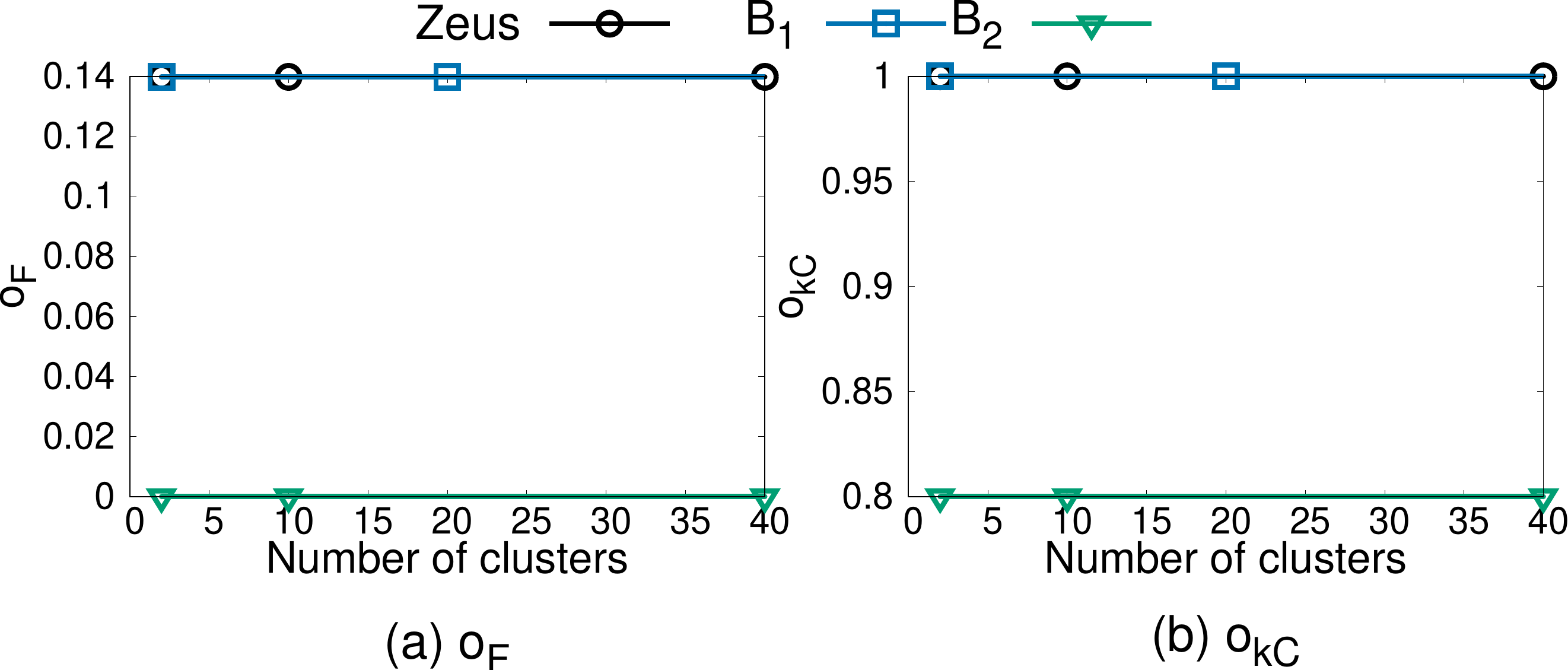}
	\caption{Results on modified Fairness objective.\label{fig:f2_merged}}
	\vspace{-7pt}
\end{figure}
\end{document}